\documentclass[submission,creativecommons]{eptcs}
\providecommand{\event}{First International Workshop on Strategic Reasoning, Rome, 2013} 

\usepackage{times}
\usepackage{helvet}
\usepackage{courier}
\usepackage[english]{babel}
\usepackage[latin1]{inputenc}
\usepackage{amsmath}
\usepackage{amssymb}
\usepackage{amsfonts}
\usepackage{amsthm}
\usepackage{mathrsfs}

\usepackage{verbatim}

\usepackage{multicol}
\usepackage{booktabs}

\usepackage{graphicx}

\title{Restricted Manipulation in Iterative Voting: \\Convergence and Condorcet Efficiency} 

\author{
Umberto Grandi
\institute{University of Padova}
\email{umberto.uni@gmail.com}
\and
Andrea Loreggia
\institute{University of Padova}
\email{andrea.loreggia@gmail.com}
\and
Francesca Rossi
\institute{University of Padova}
\email{frossi@math.unipd.it}
\and
Kristen Brent Venable 
\institute{Tulane University and IHMC}
\email{kvenabl@tulane.edu}
\and
Toby Walsh
\institute{NICTA and UNSW}
\email{toby.walsh@nicta.com.au}
}

\newcommand{\I}{\mathcal I}
\newcommand{\X}{\mathcal X}
\renewcommand{\b}{\mathbf b}
\newcommand{\Mo}{{\text{M1}}}
\newcommand{\Mt}{{\text{M2}}}

\newtheorem{definition}{Definition}
\newtheorem{theorem}{Theorem}

\def\titlerunning{Restricted Manipulation in Iterative Voting: Convergence and Condorcet Efficiency}
\def\authorrunning{U. Grandi, A. Loreggia, F. Rossi, K. B. Venable and T. Walsh}

\usepackage[compact]{titlesec}

\begin{document}
\maketitle
\

\begin{abstract}
\begin{quote}

In collective decision making, where a voting rule is used to take a collective decision among a group of agents, manipulation by one or more agents is usually considered negative behavior to be avoided, or at least to be made computationally difficult for the agents to perform. 
However, there are scenarios in which a restricted form of manipulation can instead be beneficial. In this paper we consider the iterative version of several voting rules, where at each step one agent is allowed to manipulate by modifying his ballot according to a set of restricted manipulation moves which are computationally easy and require little information to be performed. 
We prove convergence of iterative voting rules when restricted manipulation is allowed, and we present experiments showing that most iterative voting rules have a higher Condorcet efficiency than their non-iterative version.

%

\end{quote}
\end{abstract}


\section{Introduction}

In multi-agent systems, often agents need to take a collective decision. 
A voting rule can be used to decide which decision to take, mapping the agents' preferences over the possible candidate decisions into a winning decision for the collection of agents.
In this kind of scenarios, it seems desirable that agents do not have any incentive to manipulate, that is, to misreport their preferences in order to influence the result of the voting rule in their favor.

Manipulation is indeed usually seen as bad behavior from agents, to be avoided or at least to be made computationally difficult to accomplish. While we know that every voting rule is manipulable when no domain restriction is imposed on the agents' preferences (such as single-peakedness), we can try to make sure that a voting rule is computationally difficult to manipulate for single agents or coalitions of agents.

In this paper we consider a different setting, in which instead manipulation is allowed in a fair way.
As in the usual case, we start with agents expressing their preferences over a set of candidates and the voting rule selecting the current winner. 
However, this is just a temporary winner, since at this point a single agent may decide to manipulate, that is, to change her preference if by doing so the result changes in her favor. The process repeats with a new agent manipulating until we eventually reach a convergence state, i.e., a profile where no single agent can get a better result by manipulating. We call such a process \emph{iterative voting}.
In this scenario, manipulation can be seen as a way to achieve consensus, to give every agent a chance to vote strategically (a sort of fairness), and to account for inter-agent influence over time.

A practical example of this process is Doodle,\footnote{http://doodle.com/} a very popular
on-line system to select a time slot for a meeting by considering the preferences of the 
participants.
In Doodle, each participant can approve as many time slots as she wants, and the winning time slot is the one with the largest number of approvals. At any point, each participant can modify her vote in order to get a better result, and this can go on for several steps.
Depending on the voting rule, on the tie-breaking rule (to be used when there are several tied winners), and on the possible manipulation moves (that is, how agents 
are allowed to change their preferences in a single step), we may get convergence or not.
We will say that the iterative version of a voting rule converges if it gets to a stable state no matter the initial profile.
 
Iterative voting has been the subject of numerous publications in recent years. Previous work has focused on iterating the plurality rule \cite{MeirEtAl2010} and on the problem of convergence for several voting rules~\cite{LevRosenscheinAAMAS2012}.
Lev and Rosenschein~\cite{LevRosenscheinAAMAS2012} showed that, if we allow agents to manipulate in any way they want (i.e.,~to provide their best response to the current profile), then the iterative version of most voting rules do not converge. 
Therefore, an interesting problem is to seek restrictions on the manipulation moves to guarantee convergence of the associated iterative rule. 
Restricted manipulation moves are good not only for convergence, but also because 
they can be easier to accomplish for the manipulating agent. In fact, contrarily to what we aim for in classical voting scenarios, here we do not want manipulation to be computationally difficult to achieve. It is actually desirable that the manipulation move be 
easy to compute while not requiring too much information to be computed.

An example of a restricted manipulation move is the one for agents called $k$-pragmatists by Reijngoud and Endriss \cite{ReijngoudEndrissAAMAS2012}:
a $k$-pragmatist just needs to know the top $k$ candidates in the collective candidate order, and will move the most preferred 
of those candidates to the top position of her preference. 
To accomplish this move, a k-pragmatist needs very little information and it is computationally easy to perform the move. This move assures convergence with a number of voting rules.

In this paper we introduce two restricted manipulation moves within the scenario of iterative voting and we analyze some of their theoretical and practical properties. 
Both manipulation moves we consider are polynomial to compute and 
require little information to be used.
We show that convergence is guaranteed under both moves, except for STV for which we only have experimental evidence of convergence.
Moreover, we show that if a voting rule satisfies some axiomatic properties, such as Condorcet consistency or unanimity, then its iterative version will also satisfy the same properties as well.
For voting rules that are not Condorcet consistent, we tested experimentally whether their Condorcet efficiency (that is, the probability to elect the Condorcet winner) improves by adopting the iterative version. Our experiments show that the Condorcet efficiency improves when restricted manipulation moves are used.

The paper is organized as follows. In Section~\ref{sec:setting} we introduce the basic definitions of iterative voting and we define two new restricted manipulation moves. Section~\ref{sec:theoretical} contains theoretical results on convergence and preservation of axiomatic properties, and in Section~\ref{sec:experimental} we present our experimental evaluation of restricted iterative voting. 
Section~\ref{sec:conclusions} contains our conclusions and directions for future research.

%


\section{Background Notions}\label{sec:setting}


In this section we recall the basic notions of voting theory that we shall use in this paper, we present the setting of iterative voting, and we define a number of restrictions on the manipulation moves that agents can perform.

\subsection{Voting Rules}

Let $\X$ be a finite set of $m$ candidates and $\I$ be a finite set of $n$ individuals. 
We assume individuals have preferences $p_i$ over candidates in $\X$ in the form of \emph{strict linear orders}, i.e., transitive, anti-symmetric and complete binary relations. Individuals express their preferences in form of a ballot $b_i$ (e.g., the top candidate, a set of approved candidates, or the full linear order) and we call the choice of a ballot for each individual a profile $\b=(b_1,\dots,b_n)$. 
In this paper, we assume that individuals submit as a ballot for the election their full linear order, and we thus use the two notions of ballot and preference interchangeably. 

A (non-resolute) \emph{voting rule} $F$ associates with every profile $\b =(b_1,\dots,b_n)$ a non-empty subset of winning candidates $F(\b)\in 2^\X\setminus \emptyset$. 
There is a wide collection of voting rules that have been defined in the literature \cite{BramsFishburn2002} and here we focus on the following ones:

\begin{description}

\item \emph{Positional scoring rules (PSR)}: Let $(s_1, \dots, s_m)$ be a scoring vector such that $s_1 \geq \dots \geq s_m$ and $s_1 > s_m$.
If a voter ranks candidate $c$ at $j$-th position in her ballot, this gives $s_j$ points to the candidate. The candidates with the highest score win. 
We focus on four particular PSR: \emph{Plurality} with scoring vector $(1,0,\dots,0)$, \emph{veto} with vector $(1,\dots ,1,0)$, \emph{$2$-approval} with vector $(1,1,0,\dots,0)$, \mbox{\emph{$3$-approval}} with vector $(1,1,1,0,\dots,0)$, and \emph{Borda} with vector \mbox{$(m-1,m-2,\dots, 0)$.}

\item \emph{Copeland}: The score of candidate $c$ is the number of pairwise comparisons she wins (i.e., contests between $c$ and another candidate $a$ such that there is a majority of voters preferring $c$ to $a$) minus the number of pairwise comparisons she loses. The candidates with the highest score win.

\item \emph{Maximin}: The score of a candidate $c$ is the smallest number of voters preferring it in any pairwise comparison. The candidates with the highest score win.

\item \emph{Single Transferable Vote (STV)}: At the first round the candidate that is ranked first by the fewest number of voters gets eliminated (ties are broken following a predetermined order of candidates). Votes initially given to the eliminated candidate are then transferred to the candidate that comes immediately after in the individual preferences.
This process is iterated until one alternative is ranked first by a majority of voters.

\end{description}

\noindent
All rules considered thus far are non-resolute, i.e., they associate a set of winning candidates with every profile of preferences. To eliminate ties in the outcome we assume that the set $\X$ of candidates is ordered by $\prec_\X$, and in case of ties the alternative ranked highest by $\prec_\X$ is chosen as the unique outcome. 

\subsection{Iterative Voting}

A classical problem studied in voting theory is that of manipulation: do individuals have incentive to misreport their preferences, in order to force a candidate they prefer as winner of the election? 
The Gibbard-Satterthwaite Theorem \cite{Gibbard1973, Satterthwaite1975} showed that under natural conditions all voting rules can be manipulated.
Following this finding, a considerable amount of work has been spent on devising conditions to avoid manipulation, e.g., in form of restrictive conditions on individual preferences, or in form of computational barriers that make the calculation of manipulation strategies too hard for agents \cite{BartholdiOrlin1991,FaliszewskiProcaccia2010}. 

In this paper, we take a different stance on manipulation: 
we consider the fact that individuals are allowed to change their preferences as a positive aspect of the voting process, that may eventually lead to a better result after a sufficient number of steps.
Thus, we consider a sequence of repeated elections in which at each step one of the individuals is allowed to manipulate, i.e., to modify her ballot in order to change the outcome of the election in her favor.
The iteration process starts at $\b^0$ (which we shall refer to as the truthful profile) and continues to $\b^1,\dots,\b^k$. 
At each step only one individual $\tau(k)$ is allowed to manipulate, following a turn function $\tau$ (e.g., $\tau$ follows the order in which individuals are given), while all other individual ballots remain unchanged. 

The setting of iterative voting was first introduced and studied by \cite{MeirEtAl2010} for the case of the plurality rule, and expanded by \cite{LevRosenscheinAAMAS2012}. In their work, the authors describe the iterated election process as a \emph{voting game}, in which convergence of the iterative process corresponds to reaching a Nash equilibria of the game. 
They show that convergence is rarely guaranteed with most voting rules under consideration:
for instance, the iterative version of PSRs and Maximin do not always converge, even with deterministic tie-breaking (i.e., not randomized). 
On the other hand, plurality always converges with any tie-breaking rule, as well as veto with linear tie-breaking. 


\subsection{Restricted Manipulation Moves}

The convergence of the iterated version of a voting rule can be obtained by restricting the set of manipulation strategies available to the agents.
We now list a number of restrictions that have been studied in the literature, and we add two new definitions to this list.
Let $\b^k$ be the current profile at step~$k$, $\b^0$ be the initial (truthful) profile, and $F$ be a voting rule. Assume that $\tau(k)=i$.

\begin{description}
\item \emph{Best response} (no restriction): the manipulator $i$ changes her full ballot by selecting the linear order which results in the best possible outcome for her truthful preference $b^0_i$ \cite{LevRosenscheinAAMAS2012}.

\item \emph{k-pragmatist}: the manipulator $i$ moves to the top of her reported ballot the most preferred candidate following $b^0_i$ among those that scored in the top $k$ positions \cite{ReijngoudEndrissAAMAS2012}. 

\item \emph{M1}: the manipulator $i$ moves to the top of her reported ballot the second-best candidate in $b^0_i$, unless the current winner $w=F(\b^k)$ is already her best or second-best candidate in $b^0_i$.

\item \emph{M2}: the manipulator $i$ moves to the top of her reported ballot the most preferred candidate in $b_i^0$ which is above $w=F(\b^k)$ in $b_i^k$, among those that can become the new winner of the election.
\end{description}

\noindent
Different restrictions on manipulation moves induce different iterated versions of a voting rule:

\begin{definition}
Let $F$ be a voting rule and $M$ a restriction on manipulation moves. $F^{M,\tau}$ associates with every profile $\b$ the outcome of the iteration of $F$ using turn function $\tau$ and manipulation moves in $M$ if this converges, and $\uparrow$ otherwise. 
\end{definition}

\noindent
In the sequel we shall omit the superscript $\tau$ from the notation when this will be clear from the context. Observe that if $M$ is the set of best responses, then $F^M=F^*$.

Restrictions on the set of manipulation moves can be evaluated following three parameters: $(i)$ the convergence of the iterated voting rule associated with the restriction, $(ii)$ the information to be provided to voters for computing their strategy\footnote{This parameter is called \emph{poll information function} by Reijngoud and Endriss~\cite{ReijngoudEndrissAAMAS2012}.}, and $(iii)$ the computational complexity of computing the manipulation move at every step. An ideal restriction always guarantees convergence, requires as little information as possible, and is  computationally easy to compute.
 
As we pointed out at the end of the previous section, convergence is not guaranteed in most cases if the set of manipulation moves is not restricted (i.e., using best responses). Reijngoud and Endriss~\cite{ReijngoudEndrissAAMAS2012} show convergence for PSRs using the $k$-pragmatist restriction, and we shall investigate convergence results for $\Mo$ and $\Mt$ in the following section. 
Let us move to the other two parameters: on the one hand, $\Mo$ requires as little information as possible to be computed, i.e., only the winner of the current election, and is also very easy to compute. 
On the other hand, computing the best response requires an agent to have full knowledge of a profile, and may be computationally very hard to compute \cite{BartholdiOrlin1991}.
The $k$-pragmatist restriction has good properties: it is easy to compute, and the information required to compute the best strategy is just the set of the candidates ranked in the top $k$ positions.
$\Mt$ also requires little information for the agents: the scoring vector of candidates in case of scoring rules, the majority graph for Copeland and Maximin. In the case of STV the full profile is instead required. 
Moreover, from the point of view of the manipulator, $\Mt$ is computationally easy (i.e., polynomial) to perform.


\section{Convergence and Axiomatic Properties}\label{sec:theoretical}


In this section we prove that the iterated version of PSR, Maximin and Copeland converge when using our two new restrictions on the manipulation moves. 
We also analyze, for a number of axiomatic properties, the behavior of the iterated version of a voting rule.
\begin{theorem}\label{thm:M1}
$F^{\Mo}$ converges for every voting rule $F$. 
\end{theorem}

\begin{proof}
The proof of this statement is straightforward from our definitions. 
The iteration process starts at the truthful profile $\b_0$, and each agent is then allowed to switch the top candidate with the one in second position.
Thus, the iteration process stops after at most $n$ steps.
\end{proof}


\begin{theorem}\label{thm:M2}
$F^{\Mt}$ converges if $F$ is a PSR, the Copeland rule or the Maximin rule. 
\end{theorem}

\begin{proof}
The winner of an election using a PSR, Copeland or Maximin is defined as the candidate maximizing a certain score (or with maximal score and higher rank in the tie-breaking order). 
Since the maximal score of a candidate is bounded, it is sufficient to show that the score of the winner increases at every iteration step (or, in case the score remains constant, that the position of the winner in the tie-breaking order increases) to show that the iterative process converges.

Let us start with PSR. Recall that the score of a candidate~$c$ under PSR is $\sum_i s_i$ where $s_i$ is the score given by the position of $c$ in ballot $b_i$. 
Using $\Mt$, the manipulator moves to the top a candidate which lies above the current winner~$c$. Thus, the position -- and hence the score -- of $c$ remains unchanged, and the new winner must have a strictly higher score (or a better position in the tie-breaking order) than the previous one. 
The case of Copeland and Maximin can be solved in a similar fashion: 
it is sufficient to observe that the relative position of the current winner $c$ with all other candidates (and thus also its score) remains unchanged when ballots are manipulated using $\Mt$. 
Thus, the Copeland score and the Maximin score of a new winner must by higher than that of $c$ (or the new winner must be placed higher in the tie-breaking order). 
\end{proof}

\noindent 
While currently we do not have a proof of convergence for STV, we observed experimentally that its iteration always terminates on profiles with a Condorcet winner when a suitable turn function described in the following section is used. 

Voting rules are traditionally studied using axiomatic properties, and we can inquire whether these properties extend from a voting rule to its iterated version. We refer to the literature for an explanation of these properties \cite{Taylor2005}. 
Let us call $F_t^M$ the iterated version of voting rule $F$ after $t$ iteration steps. 
We say that a restricted manipulation move $M$ \emph{preserves} a given axiom if whenever a voting rule $F$ satisfies the axiom then also $F_t^M$ does satisfy it for all $t$.

\begin{theorem}
$\Mo$ and $\Mt$ preserve unanimity.
\end{theorem}

\begin{proof}
Assume that the iteration process starts at a unanimous profile $\b$ in which candidate $c$ is at top position of all individual preferences. 
If $F$ is unanimous, then $F(\b)=c$, and no individual has incentives to manipulate either using $\Mo$ or $\Mt$. Thus, iteration stops at step~1 and $F_t^\Mo(\b)=c$ and $F_t^\Mt(\b)=c$, satisfying the axiom of unanimity. 
\end{proof}

\begin{theorem}
$\Mo$ and $\Mt$ preserve Condorcet consistency.
\end{theorem}

\begin{proof}
Let $c$ be the Condorcet winner of a profile $\b$. If $F$ is Condorcet-consistent then $F(\b)=c$.
As previously observed, when individuals manipulate using either $\Mo$ or $\Mt$ the relative position of the current winner with all other candidates does not change, since the manipulation only involves candidates that lie above the current winner in the individual preferences. 
Thus $c$ remains the Condorcet winner in all iteration steps $\b^k$. Since $F^\Mo_k(\b)=F(\b^k)$ and F is Condorcet-consistent, we have that $F^\Mo_k(\b)=c$ and thus $F^\Mo_k$ is Condorcet consistent. Similarly for $\Mt$. 
\end{proof}

\noindent
Other properties that transfer from a voting rule to its iterative version are neutrality and anonymity (supposing the turn function satisfies an appropriate version of neutrality and anonymity). 
The Pareto-condition does not transfer to the iterated version, as can be shown by adapting an example by Reijngoud and Endriss~\cite{ReijngoudEndrissAAMAS2012}.


\section{Experimental Evaluation of Restricted Manipulation Moves}\label{sec:experimental}


In this section we evaluate our two restricted manipulation moves $\Mo$ and $\Mt$ under one important aspect: we measure whether the restricted iterative version of a voting rule has a higher Condorcet efficiency than the initial voting rule, i.e., whether the probability that a Condorcet winner (if it exists) gets elected is higher for the iterative rather than non-iterative rule. 
We show that in most cases the Condorcet efficiency of a voting rule increases if iterated manipulation is allowed using $\Mo$ or $\Mt$ (except for Copeland and Maximin which are already Condorcet-consistent rules). 
We also compare our findings with the \emph{k-pragmatist} restriction for $k=2,3$.

Our results are obtained using a program implemented in Java ver.1.6.0. 
The software generates profiles with uniform distribution (i.e., impartial culture assumption). 
The impartial culture assumption has received criticism in recent years \cite{Regenwetter}. However, it remains the most common assumption used in social choice theory, and thus represents the obvious starting point for our empirical evaluation.
Our test set contains 10.000 profiles with Condorcet winner. We set the number of candidates to $5$ and varied the number of voters from $20$ to $100$.

The turn function used in our experiments associates to each voter $i$ a dissatisfaction index $d_i(k)$, which increases of one point for each iteration step in which the individual has an incentive to manipulate but is not allowed to do so by the turn function. At iteration step $k$ the individual that has the highest dissatisfaction index is allowed to move (in the first step, and in case of ties, the turn follows the initial order in which voters are given).
We were always able to compute the outcome of the iterative voting rules after convergence in reasonable time. 

\subsection{Condorcet efficiency of restricted iterative voting}


Figure~\ref{figure:nuovografico} compares, for several voting rules, the Condorcet efficiency of the respective iterative version using restricted manipulation moves $\Mo$, $\Mt$, 2-\emph{pragmatist} and 3-\emph{pragmatist}. The number of voters is set to $n=50$.

\begin{figure}[htps]
\begin{center}
\includegraphics[width=0.46\textwidth]{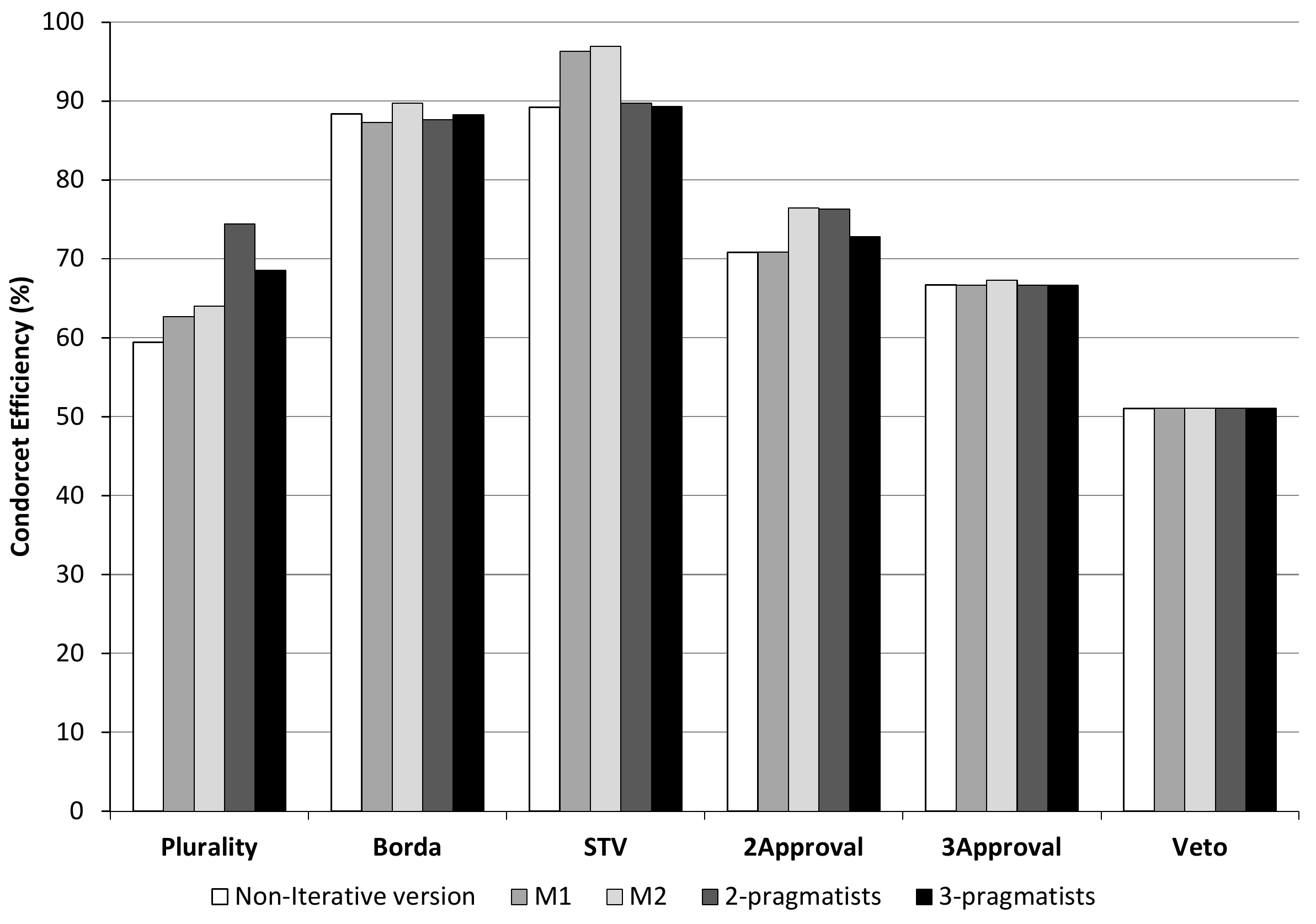}\\

\caption{Iterated Condorcet efficiency.}
\label{figure:nuovografico}
\end{center}
\end{figure}

\noindent
Except for the case of the Borda rule, the Condorcet efficiency of the iterated version of a voting rule improves significantly with respect to the non-iterated version, and the growth is significantly higher when voters manipulate the election using $\Mt$ rather than $\Mo$.
A plausible reason for this behavior is the difference in range of candidates that can be helped by the two manipulation moves. While $\Mo$ may help a Condorcet winner being elected only if it was ranked second by some of the individuals, $\Mt$ may help a candidate even if it was ranked lower. 
Let us also stress that while the increase in Condorcet efficiency using $\Mo$ is minimal, it is still surprising that such a simple move can result in a better performance than the original version of the voting rule. 
The 2-\emph{pragmatist} and 3-\emph{pragmatist} restriction perform quite well with the plurality rule, while for all other rules our restriction $\Mt$ results in a better performance.

In the case of Plurality a significant increase can be obtained with both $\Mo$ and $\Mt$. In Figure~\ref{figure:plurality} we show the trend in Condorcet efficiency when voters vary from $n=20$ to $n=100$. It can be observed that the increase is higher for smaller numbers of voters and stabilizes at around $n=60$. The same behavior can be observed in Figure~\ref{figure:STV} for the case of STV.

\begin{figure}[htps]
\parbox{.50\linewidth}{
\centering
\vspace{1.2cm}

\begin{tabular}{c}
\includegraphics[width=0.46\textwidth]{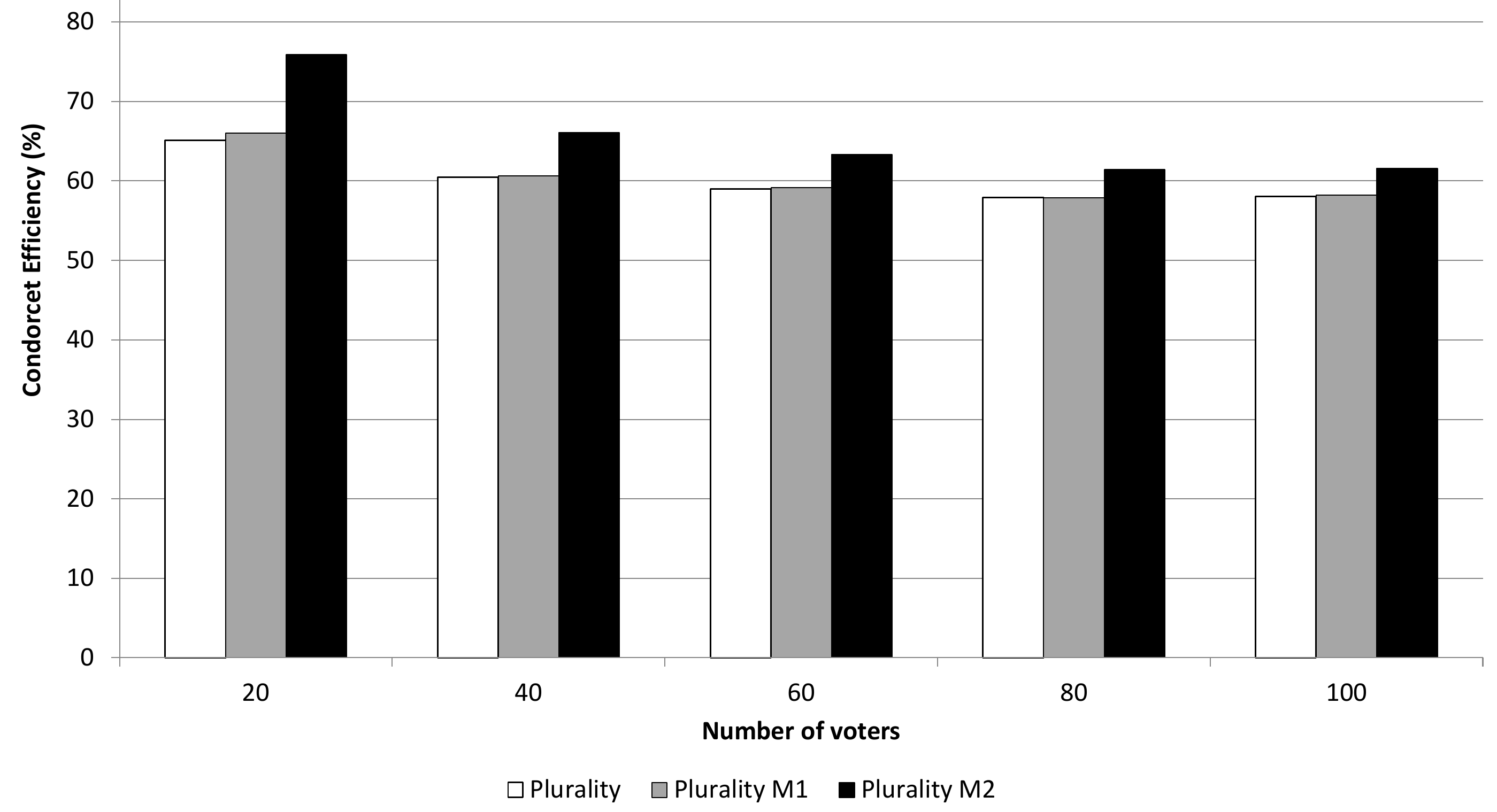}\\
\end{tabular}
\caption{Condorcet efficiency of plurality.}
\label{figure:plurality}
}
\hfill
\parbox{.50\linewidth}{
\centering

\begin{tabular}{c}

\includegraphics[width=0.46\textwidth]{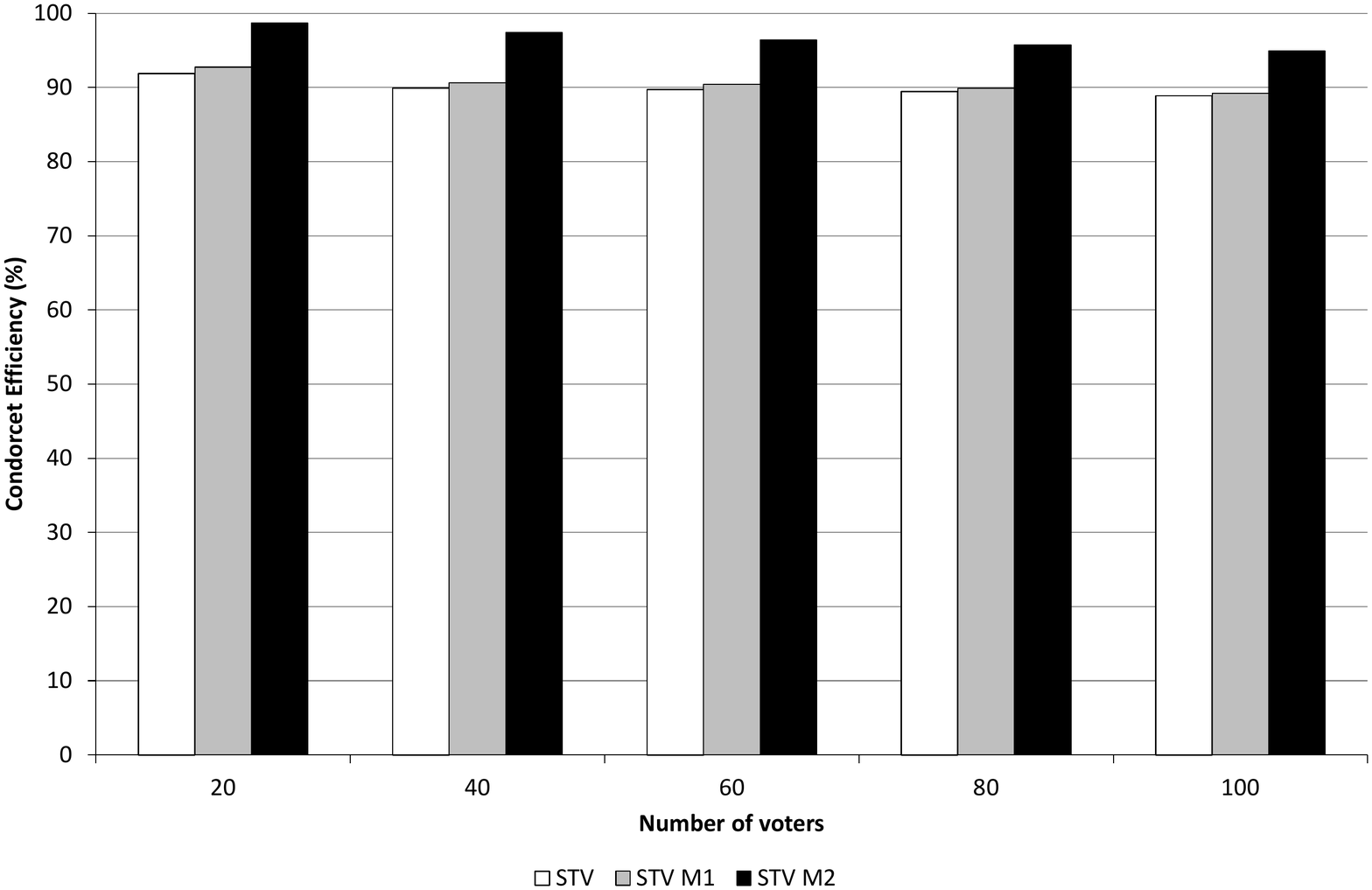}
\end{tabular}
\caption{Condorcet efficiency for iterated STV.}
\label{figure:STV}

}
\end{figure}

\noindent
STV has the highest performance of all voting rules considered thus far. 
STV has already a high Condorcet efficiency, but this is amplified by the use of manipulation moves, in particular $\Mt$. 
In Figure~\ref{figure:STV} we show that its Condorcet-efficiency can be augmented to more than 95 percent. As remarked earlier, we observed convergence in all profiles considered.


The absence of any increase in Condorcet efficiency for veto (as well as 2-approval and 3-approval using $\Mo$) is a consequence of the fact that our restricted moves do not change the candidates' score with these particular scoring vectors.

\section{Conclusions and Future Work}\label{sec:conclusions}


This paper studies the iteration of classical voting rules allowing individuals to manipulate the outcome of the election using a restricted set of manipulation moves.

We provided two new definitions of manipulation moves $\Mo$ and $\Mt$ and showed that they lead to convergence for all voting rules considered (cf. Theorem \ref{thm:M1} and \ref{thm:M2}). 
We also showed that most axiomatic properties, such as unanimity and Condorcet consistency, are preserved in the iteration process. 
We evaluated the performance of our restricted manipulation moves with respect to the Condorcet efficiency of the iterated version of a voting rule as well as the average position of the winner in the initial truthful profile. 
Our experiments showed that allowing restricted manipulation in iterative voting yields a positive increase in Condorcet efficiency, and that, predictably, the best performance is obtained when more information is given to agents (cf. the case of STV with $\Mt$). 


This work gives rise to a number of interesting directions to be explored in future research. 
First, different restrictions on manipulation moves may be considered, and their performance should be compared with that of existing definitions. We tested a move similar to $\Mt$, which did not restrict the choice of a candidate to those who become the new winners of the iterated election, obtaining a performance comparable to that of $\Mt$.
Restricted manipulation moves may also be evaluated using other parameters, and could be tested on more realistic distributions of profiles of preferences, for instance by exploiting data extracted from Internet-based polling services like Doodle.

{\small
\bibliographystyle{eptcs}
\bibliography{iterated}
}
\end{document}